\documentclass[journal]{IEEEtran}

\usepackage{cite}
\ifCLASSINFOpdf
\else
\fi
\usepackage{amsmath}
\usepackage{amssymb}
\usepackage{amsthm}

\usepackage{algpseudocode}
\usepackage{algorithm}

\usepackage{graphicx}

\DeclareMathOperator\T{\mathsf{T}}
\DeclareMathOperator\E{\mathsf{E}}

\DeclareMathOperator\var{\mathsf{Var}}

\DeclareMathOperator*{\argmax}{arg\,max}

\newtheorem{theorem}{Theorem}
\newtheorem{lemma}{Lemma}

\begin{document}
\title{Localization Algorithm with Circular Representation in 2D and its Similarity to Mammalian Brains}

\author{Tsang-Kai Chang, Shengkang Chen,
        and Ankur Mehta
\thanks{The authors are with the Department
of Electrical and Computer Engineering, University of California, Los Angeles,
CA, 90095 USA. e-mail: \{tsangkaichang, billyskc, mehtank\}@ucla.edu}% <-this % stops a space
}%\thanks{Manuscript received April 19, 2005; revised August 26, 2015.}}


\markboth{IEEE Robotics and Automation Letters}%
{Chang \MakeLowercase{\textit{et al.}}: Bare Demo of IEEEtran.cls for IEEE Journals}



\maketitle

\begin{abstract}
   Extended Kalman filter (EKF) does not guarantee consistent mean and covariance under linearization, even though it is the main framework for robotic localization. While Lie group improves the modeling of the state space in localization, the EKF on Lie group still relies on the arbitrary Gaussian assumption in face of nonlinear models. We instead use von Mises filter for orientation estimation together with the conventional Kalman filter for position estimation, and thus we are able to characterize the first two moments of the state estimates. Since the proposed algorithm holds a solid probabilistic basis, it is fundamentally relieved from the inconsistency problem. Furthermore, we extend the localization algorithm to fully circular representation even for position, which is similar to grid patterns found in mammalian brains and in recurrent neural networks. The applicability of the proposed algorithms is substantiated not only by strong mathematical foundation but also by the comparison against other common localization methods.
\end{abstract}

%\begin{IEEEkeywords}
%   Localization, biologically-inspired robots, and probability and statistical methods.
%\end{IEEEkeywords}

\section{Introduction}

% autonomy, EKF is the core
   Localization enables robotic autonomy, in which robots realize their own spatial state from both proprioceptive and exteroceptive information in order to accomplish high-level tasks. Since noise is ubiquitous in obtained information, localization depends on estimation algorithms to separate information from noise. Among all estimation algorithms, extended Kalman filter (EKF) plays a critical role in localization, even in the realization of simultaneous localization and mapping (SLAM) \cite{cadena_past_2016}.
   Nevertheless, since time propagation and observation models are inevitably nonlinear in localization, the Gaussian assumption of the state estimates in EKF is definitely violated. It is also well-known that linearization in EKF results in inconsistency problem \cite{julier_counter_2001, bailey_consistency_2006, huang_convergence_2007}. In short, the first two moments from EKF are not necessarily equal to those of the output estimates, let alone the output distribution. Moreover, without knowing the ground truth values, the evaluation of Jacobian at the estimated values as approximation in EKF only worsens the estimation credibility.

% Lie group solution
   The major breakthrough of the inconsistency problem comes from Lie groups, with thorough treatment in   \cite{brian_c_hall_lie_2015}. The spatial state of localization in 2D, including position and orientation, is better described by the $SE(2)$ Lie group rather than  Euclidean space. Consequently, the observers in Lie group are designed, especially the extended Kalman filter on Lie group (LG-EKF) \cite{bourmaud_discrete_2013}. LG-KEF has been applied in deterministic nonlinear observer \cite{barrau_invariant_2017}, and also in SLAM \cite{zhang_convergence_2017, brossard_invariant_2018}. LG-EKF relies on concentrated Gaussian distribution, which is defined on the associated Lie algebra with small variation. However, the exact distribution on the Lie group, where the spatial state really lies, is unclear, which hinders direct application. In addition, the assumption of Gaussian distribution on Lie algebra can hardly be maintained due to the nonlinear time propagation and observation models, and thus the small variation assumption should always be ensured for applicability.
   
% our approach, circular representation
   In this paper, we construct a localization algorithm with mixture state representation, which consists of circular representation for orientation and Euclidean one for position. By applying von Mises filter for orientation estimate, we are no longer limited to the Gaussian state estimates, and can at least calculate the first two moments of the state estimates in face of nonlinear time propagation and observation models. In particular, we use both Kalman filter and von Mise filter jointly in the localization algorithm. Moreover, apart from the evaluation of Jacobian in EKF, no ground truth value is needed in the proposed algorithm.

% similarity to biology
   Prior to the application in robotics, circular representation already appears in mammals, who perform SLAM continually and successfully. With decades of investigation in neuroscience, several neurons are found with spatial selectivity, including place cells, head direction cells and grid cells. In particular, grid cells spike when the animal is at points arranged in a hexagonal grid pattern \cite{hafting_microstructure_2005, mcnaughton_path_2006}. The grid pattern also emerges in the artificial neurons of recurrent networks conducting navigation tasks \cite{banino_vector_based_2018}. Since the grid pattern is context-independent, it represents the universal spatial structure. While head direction cells encode orientation information in circular form, grid cells with various spatial scales together provide position information, also in circular form. Therefore, we can establish a localization algorithm with fully circular representation, which is similar to mammalian brains. 
   In \cite{milford_mapping_2008}, an SLAM architecture inspired by grid cells is proposed. However, no proprioceptive information is integrated in this system, which deviates from the grid cell mechanism. Instead, the proposed algorithms in this paper are derived along with faithful neuroscientific understanding, in order to provide a unified principle for localization.

% contribution
   The contributions of this paper include:
   \begin{itemize}
      %\item the complete characterization of concentration parameter in the time update of von Mises filter,
      \item the localization algorithm with mixture representation which preserves the first two moments of estimates,
      \item the extension to localization algorithm with fully circular representation as in mammalian brains, and
      \item the discussion of potential investigation direction between robotics and neuroscience.
   \end{itemize}

% structure
   This paper is organized as follows: The background of circular distribution and filtering is presented in Section II. The localization algorithm with mixture representation is presented in Section III. In Section IV, the spatial representation in mammalian brains in reviewed together with the proposal of the localization algorithm with fully circular representation. Simulation examples are given in Section V. We discuss the results of this paper in Section VI, and conclude this paper in the last section.

\section{Circular Distribution and Filtering}

   There are two widely used circular distributions applied in circular filtering: von Mises distribution and wrapped normal distribution \cite{mardia_directional_1999}. The circular filter based on von Mises distribution, or simply von Mises filter, was developed in \cite{azmani_recursive_2009}, and further extended to distribution for higher dimensional spheres in \cite{markovic_multitarget_2016}. As for wrapped normal distribution, the corresponding filter is summarized in \cite{kurz_recursive_2016}. In addition, the intimate relationship between two distribution in filtering problem is highlighted in \cite{kurz_recursive_2016}. 
   
   In this paper, we focus on von Mises formulation since von Mises distribution is closed under multiplication, which is the necessary operation for observation update in filter. Even though von Mises distribution is not closed under circular convolution, the approximate time update in von Mises filter can preserve the trigonometric moments \cite{kurz_recursive_2016}. We summarize von Mises distribution and filter in the following for completeness, and add one theorem to substantiate the effect of time update.

%%%%%
\subsection{von Mises Distribution}

   The von Mises distribution, denoted by $vM(\mu,\kappa), \kappa > 0$, has probability density function
   \begin{equation}
      g(\theta;\mu,\kappa)=\frac{1}{2\pi I_0(\kappa)} e^{\kappa\cos(\theta-\mu)},\quad  0 \leq \theta < 2\pi,
   \label{eq:von_mises_pdf}
   \end{equation}
   where $I_p$ is the modified Bessel function of the first kind and order $p$, which can be defined by
   \begin{equation}
      I_p(\kappa) = \frac{1}{2\pi} \int_0^{2\pi}\cos( p\theta) e^{\kappa \cos \theta} d\theta.
   \end{equation}
   In (\ref{eq:von_mises_pdf}), $\mu$ is the mean direction and $\kappa$ is known as the concentration parameter. With sufficiently large $\kappa$, the von Mises distribution $vM(\mu,\kappa)$ resembles the Gaussian distribution with mean $\mu$ and variance $\frac{1}{\kappa}$, denoted by $N(\mu,\frac{1}{\kappa})$. This approximation will be used later as the connection between Gaussian and von Mises distributions.
   
   The connection between von Mises and Gaussian distributions also arises in conditioning, where the von Mises can be generated by conditioning on bivariate Gaussian distribution.
   \begin{theorem}
      Let $v$ be bivariate random vector with mean $d[\cos\phi, \sin\phi]^{\T}$ and covariance $\sigma^2 I_2$, where $I_n$ is a $n\times n$ identity matrix. With the expression $v=r[\cos\theta, \sin\theta]^{\T}$, the conditional probability of $\theta$ given $r=r_0$ is $vM(\phi, r_0 d/\sigma^2)$.
      \label{theorem:vonMises_generate}
   \end{theorem}

   As von Mises random variable often appears in trigonometric functions, the property of the trigonometric moment for von Mises distribution is essential, and is summarized in the following lemma.
   \begin{lemma}[\cite{kurz_recursive_2016}]
   For random variable $\theta$ following $vM(\mu,\kappa)$, the $n$th trigonometric moments is given by
   \begin{equation}
      \E[\exp(i n \theta)] =  \exp(i n \mu) \frac{I_{n}(\kappa)}{I_0(\kappa)}.
   \end{equation}
   \label{lemma:vM_moment}
   \end{lemma}

%%%%%
\subsection{von Mises Filter}

   \begin{figure}[t]
      \centering
      \includegraphics[width=0.4\textwidth]{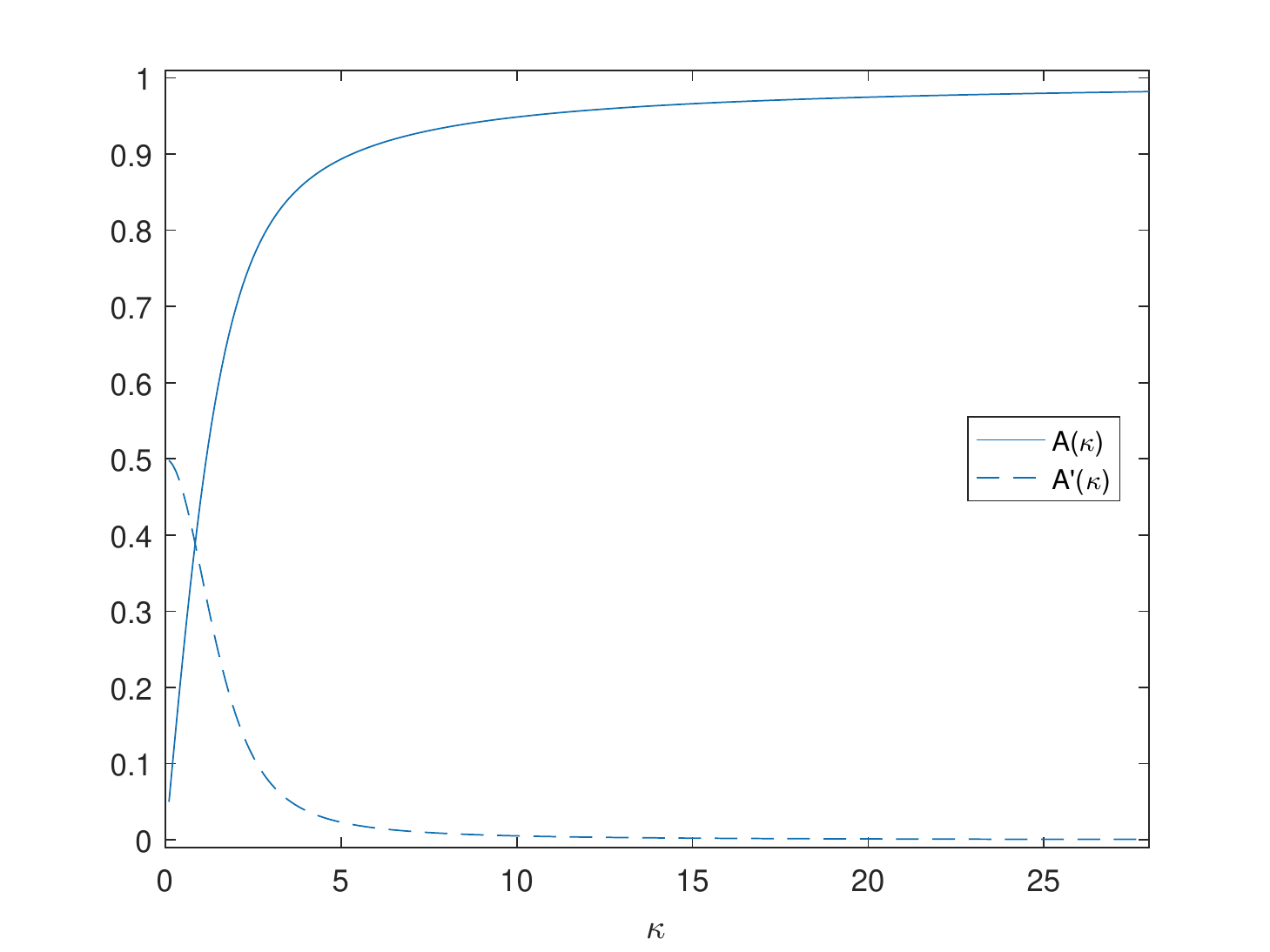}
      \caption{The plot of $A(\kappa)$ and $A'(\kappa)$ defined in (\ref{eq:a_func}).}
      \label{fig:a_func}
   \end{figure}   
   
   % state space model
   We consider a discrete-time system with time index $t$ to describe the evolution of circular state $\theta_t \in [0,2\pi)$ as
   \begin{equation}
      \theta_{t+1} = \theta_t + u_t + w_t,
      \label{eq:time_model}
   \end{equation}
   where $u_t$ is the input at time $t$ and $w_t$ is the process noise, which is independent of the state and is modeled by $vM(0,\kappa_w)$. 
   To estimate the state, the state $\theta_t$ can be observed by the observation model
   \begin{equation}
      o_t = \theta_t + \nu_t,
      \label{eq:observation_model}
   \end{equation}
   where $\nu_t$ is the observation noise modeled by $vM(0,\kappa_\nu)$. Based on the principle in Kalman filter, we can use $\hat{\theta}_t$, modeled by $vM(\mu_t,\kappa_t)$, to estimate $\theta_t$, and update $\hat{\theta}_t$ recursively by the input $u_t$ and observation $o_t$. 
   
   % time update
   For time update, given the current estimator $\hat{\theta}_t$, the updated estimator $\hat{\theta}_{t+1}$ is the sum of two independent von Mises random variables. Let $\theta_1$ and $\theta_2$ be independently distributed as $vM(\mu_1,\kappa_1)$ and $vM(\mu_2,\kappa_2)$, respectively. The probability density function of $\theta=\theta_1+\theta_2$ is given by
   \[
      h(\theta) = \frac{1}{4\pi^2 I_0(\kappa_1) I_0(\kappa_2)} \int_0^{2\pi} e^{\kappa_1 \cos(\xi-\mu_1)+\kappa_2 \cos(\theta -\xi-\mu_2)} d\xi,
   \]
   which is not a von Mises distribution. However, $h(\theta)$ can be approximated by a von Mises distribution $vM(\mu_1+\mu_2), A^{-1}(A(\kappa_1)A(\kappa_2))$, where
   \begin{equation}
      A(\kappa) = \frac{I_1(\kappa)}{I_0(\kappa)}.
      \label{eq:a_func}
   \end{equation}      
   Figure \ref{fig:a_func} depicts the numerical value of $A(\kappa)$.
   The approximation, appearing in \cite{mardia_directional_1999}, was used in \cite{azmani_recursive_2009} for the development of von Mises filter. Later in \cite{kurz_recursive_2016}, the authors present that the time update of circular filtering can be done based on trigonometric moments only, without specifying a particular circular distribution, which finally validates the approximation theoretically. Consequently, $\hat{\theta}_{t+1}$ can be approximated by the distribution $vM(\mu_{t+1},\kappa_{t+1})$, with $\mu_{t+1}=\mu_t+u$ and $\kappa_{t+1}=A^{-1}(A(\kappa_t)A(\kappa_w))$.

   % observation update
   In observation update, the posterior probability density of the state estimate given the observation $o_t$ is given by
   \[
      p(\theta_t =\theta | o_t) = \frac{p(o_t | \theta_t = \theta) p(\theta_t=\theta | \theta_{t})}{p(o_t | \theta_{t})},
   \]
   where $p(o_t | \theta_t = \theta)$ follows $vM(o_t-\theta,\kappa_\nu)$ and $p(\theta_t=\theta | \theta_{t})$ follows $vM(\mu_{t}, \kappa_{t})$.  With the multiplication of von Mises distribution in the nominator, we have
   \[
      p(\theta_t =\theta | o_t) \propto e^{\kappa_\nu \cos(o_t-\theta) + \kappa_{t} \cos(\theta-\mu_{t})}= e^{\kappa_{t^+} \cos(\theta - \mu_{t^+})},
   \]
   where
   \begin{align*}
      \mu_{t^+} &= \arg\left( \kappa_\nu e^{i\,o_t}+ \kappa_{t} e^{i\,\mu_{t}} \right),\\
      \kappa_{t^+} &= \kappa_\nu \cos(\mu_{t^+}-o_t)+ \kappa_{t} \cos(\mu_{t^+}-\mu_{t}).
   \end{align*}
   Therefore, the posterior distribution $p(\theta_t = \theta | o_t)$ is still a von Mises distribution, denoted by $vM(\mu_{t^+}, \kappa_{t^+})$, where the time index $t^+$ indicates the time instance after the observation.
   The overall estimation algorithm is summarized in Algorithm 1.

\begin{algorithm}[t]
  \caption{von Mises Filter}
  \begin{algorithmic}
    \State \textbf{Initialization}:
       Set $\mu_0$ and $\kappa_0$.
    \State \textbf{Time update}: $(\mu_{t+1},\kappa_{t+1})= \texttt{vMF\_Time}(\mu_{t},\kappa_{t};  u_t, \kappa_w)$
   \begin{align*}
      \mu_{t+1}&=\mu_{t} + u_t,\\ 
      \kappa_{t+1} &=A^{-1}(A(\kappa_t)A(\kappa_w)).
   \end{align*}    
   \State \textbf{Observation update}: \\$(\mu_{t^+},\kappa_{t^+})= \texttt{vMF\_Obsv}(\mu_{t},\kappa_{t};  o_t, \kappa_\nu)$
   \begin{align*}
      \mu_{t^+} &= \arg\left( \kappa_\nu e^{i\,o_t}+ \kappa_{t} e^{i\,\mu_{t}} \right),\\
      \kappa_{t^+} &= \kappa_\nu \cos(\mu_{t^+}-o_t)+ \kappa_{t^-} \cos(\mu_{t^+}-\mu_{t}).
   \end{align*}
  \end{algorithmic}
\end{algorithm}

\subsection{Properties of $A(\kappa)$}   

   In von Mises filter, while the approximation has to be applied to maintain the same distribution during the time update, the variation of the concentration parameter becomes obscure. Although such statement can be ensured with the introduction of intermediate wrapped normal distribution, we instead explain the effect on the concentration parameter by the following theorem.   
   \begin{theorem}
   For $\kappa_1,\kappa_2 > 0$,
   \begin{equation}
      A^{-1}\left(A(\kappa_1)A(\kappa_2)\right) < \min(\kappa_1,\kappa_2). \end{equation} \label{theorem:func_A}
   \end{theorem} 
    We present the proof of Theorem \ref{theorem:func_A} in Appendix I. Theorem \ref{theorem:func_A} states that the concentration parameter of the approximated von Mises distribution decreases after time update, which conforms to the fact that time update introduces estimation uncertainty. 
    However, different from time update, the concentration parameter does not necessarily increase after observation update.

%%%%%
\section{Localization Algorithm with Mixture Representation}

   The spatial state of an agent at time $t$ includes orientation $\theta_t$ and position $x_t,y_t$. We consider unicycle model where the evolution of state with input angular velocity $\omega_t$ and translational velocity $\nu_t$ is governed by
   \begin{align}
      \theta_{t+1} &= \theta_t + (\omega_t + n_{\omega,t}) \Delta t, \notag\\
      x_{t+1} &= x_t + (v_t + n_{v,t}) \cos\theta_t \Delta t, \notag\\
      y_{t+1} &= y_t + (v_t + n_{v,t}) \sin\theta_t \Delta t. \label{eq:l_time_model}
   \end{align}   
   In (\ref{eq:l_time_model}), $n_{\omega,t}$ and $n_{v,t}$ are the noise of angular velocity and that of translational velocity, and modeled by $N(0,\sigma_{\omega}^2)$ and $N(0,\sigma_v^2)$, respectively. The agent keeps the estimates $\hat{\theta}_t$, $\hat{x}_t$ and $\hat{y}_t$ to track the corresponding terms, where $\hat{\theta}_t$ is a von Mises random variable with distribution $vM(\bar{\theta}_t,\kappa_t)$ and $\hat{x}_{t}$, $\hat{y}_t$ follow $N(\bar{x}_t,\sigma_{x,t}^2)$ and $N(\bar{y}_t,\sigma_{y,t}^2)$, respectively.  
   
%%%
\subsection{Time Update}
   
   The time update for $\hat{\theta}_t$ is relatively easy, since we can approximate $n_{\omega,t} \Delta t$ by $vM(0,1/\sigma_{\omega}^2 \Delta t^2)$ and apply Algorithm 1 directly. However, in the time update of $\hat{x}_t$, the exact distribution of $(v_t + n_{v,t}) \cos\hat{\theta}_t$ is complicated for any practical application. Instead of craving the exact distribution, we instead characterize the first two moments of $(v_t + n_{v,t}) \cos\hat{\theta}_t$ for reasonable approximation.
   With the distribution of $\hat{\theta}_t$ denoted by $vM(\bar{\theta}_t, \kappa_t)$ and the properties of trigonometric moments of von Mises distributions in Lemma \ref{lemma:vM_moment}, we have
   \begin{align*}
      \E\left[(v_t + n_{v,t}) \cos\hat{\theta}_t\right] &= v_t \E[\cos\hat{\theta}_t] + \E[n_{v,t}] \E[\cos\hat{\theta}_t] \\
         &= v_t A(\kappa_t)\cos\bar{\theta}_t,
   \end{align*}
   and together with the law of total variance
   \begin{align*}
      \var\left[(v_t + n_{v,t}) \cos\hat{\theta}_t\right] &= \E[\sigma_v^2 \cos\hat{\theta}_t ] + \var(v_t \cos\hat{\theta}_t) \\
         &\leq \sigma_v^2 + v_t^2.
   \end{align*}
   Although the random variable $(v_t + n_{v,t}) \cos\hat{\theta}_t$ is not Gaussian, we proceed to approximate it as a Gaussian random variable with mean $v_t A(\kappa_t)\cos\bar{\theta}_t$ and variance assigned as the upper bound $\sigma_v^2 + v_t^2$ to compensate the approximation discrepancy. As a consequence, we rewrite the time update model as
   \begin{align}
      x_{t+1} &\approx x_t + \left[ v_t  A(\kappa_t)\cos\bar{\theta}_t + w_{x,t} \right] \Delta t, \notag\\
      y_{t+1} &\approx y_t + \left[ v_t  A(\kappa_t)\sin\bar{\theta}_t + w_{y,t} \right] \Delta t, \label{eq:l_time_model_appr}
   \end{align}  
   where both $w_{x,t}$ and $w_{y,t}$ follows $N(0,\sigma_v^2 + v_t^2)$. The position estimates can now be updated by traditional Kalman filter.
   
%%%
\subsection{Observation Update with Direct Measurements}
   
   % observation update
   As for observation update, we first consider the direct measurement model. This straightforward model serves as be the basis for further complicate situation. In direct measurement model, the additive noises occur as
   \begin{align}
      o_{\theta,t} &= \theta_t + \nu_{\theta,t}, \notag\\
      o_{x,t} &= x_t + \nu_{x,t}, \notag\\
      o_{y,t} &= y_t + \nu_{y,t}, \label{eq:d_observ_model}
   \end{align}  
   where the orientation observation noises $\nu_{\theta,t}$ is modeled by $vM(0,\kappa_{\nu_\theta})$ and the position observation noise $\nu_{x,t}$ by $N(0,\sigma^2_{o_x})$ and  $\nu_{y,t}$ by $N(0,\sigma^2_{o_y})$. The observation update for $\hat{\theta}_t$ then simply follows the corresponding equation in Algorithm 1, while that for $\hat{x}_t$ and $\hat{y}_t$ can be done by conventional Kalman filter.

%%%
\subsection{Observation Update with Bearing and Distance Measurements}

   % model
   For most localization systems in practice, agents can hardly observe their spatial states directly. To ensure the applicability, we investigate the localization algorithm relying on bearing and distance measurements. We assume that a landmark with known position $(x_l,y_l)$ is present in the environment, and the relative bearing $b_t$ and distance $r_t$ of the landmark at time $t$ are thus given by
   \begin{align}
      b_t &= \tan^{-1} \left(\frac{y_l - y_t}{x_l - x_t} \right) - \theta_t + \nu_{b,t}, \notag \\
      r_t &= \sqrt{(x_l - x_t)^2+(y_l - y_t)^2} + \nu_{r,t}.
   \end{align}
   For convenience, we model the bearing observation noise $\nu_{b,t}$ as $vM(0,\kappa_b)$ and the distance observation noise $\nu_{r,t}$ as $N(0,\sigma_r^2)$, respectively. 
   We denote the observed bearing and distance values as $s_{b,t}$ and $s_{r,t}$, respectively.
   
   % approach
   We can reconstruct the spatial states $\theta, x, y$ from the bearing and distance observations, and leverage the aforementioned direct measurement model (\ref{eq:d_observ_model}) for update. In other words, we can construct the equivalent direct measurement given by
   \begin{align}
      o'_{\theta,t} &= \tan^{-1} \left(\frac{y_l - \hat{y}_t}{x_l - \hat{x}_t} \right) - b_t = \theta_t + \nu'_{\theta,t}, \notag \\
      o'_{x,t} &= x_l - r_t \cos(\hat{\theta}_t + b_t) = x_t + \nu'_{x,t}, \notag \\
      o'_{y,t} &= y_l - r_t \sin(\hat{\theta}_t + b_t) = y_t + \nu'_{y,t}.
      \label{eqs:observ}
   \end{align}
   In (\ref{eqs:observ}), the equivalent observation noises, $\nu'_{\theta,t}, \nu'_{x,t}$ and $\nu'_{y,t}$, not only come from the sensor noises $\nu_{b,t}, \nu_{r,t}$, but also from the state estimation uncertainty. As long as we can characterize those zero-mean equivalent observation noises, we can update the state estimates directly.
   
  % by following the direction observation model.

   % orientation
   For equivalent orientation measurement $o'_{\theta,t}$, its distribution can be characterized by analyzing $\tan^{-1} \left(\frac{y_l - \hat{y}_t}{x_l - \hat{x}_t} \right)$ and $b_t$. We know that $[x_l - \hat{x}_t,y_l - \hat{y}_t]^{\T}$ is a Gaussian random vector, and $\var(x_l - \hat{x}_t)=\var(y_l - \hat{y}_t)=\sigma^2_{x,t}$, since both estimates follow the same update equation with identical parameters in the algorithm. The distribution of $\tan^{-1} \left(\frac{y_l - \hat{y}_t}{x_l - \hat{x}_t} \right)$ is not easily determined since the covariance between $x_l - \hat{x}_t$ and $y_l - \hat{y}_t$ is not available. To circumvent this problem, we choose $2\sigma^2_{x,t}I_2$ as the nominal covariance of $[x_l - \hat{x}_t,y_l - \hat{y}_t]^{\T}$, which is guaranteed to be no smaller than the real covariance matrix in positive definite sense according to Lemma \ref{lemma:diagonal} in Appendix II. Therefore, with the nominal covariance, $ \tan^{-1} \left(\frac{y_l - \hat{y}_t}{x_l - \hat{x}_t} \right)$ follows $vM\left(\tan^{-1} \left(\frac{y_l - \bar{y}_t}{x_l - \bar{x}_t} \right), \frac{\bar{r}_t s_{r,t}}{2 \sigma^2_{x,t}} \right)$, where $\bar{r}_t=\sqrt{(x_l - \bar{x}_t)^2+(y_l - \bar{y}_t)^2}$, by Theorem \ref{theorem:vonMises_generate}. Consequently, since $\tan^{-1} \left(\frac{y_l - \hat{y}_t}{x_l - \hat{x}_t} \right)$ and $b_t$ are independent, the distribution of $\nu'_{\theta,t}$ can be well approximated by $vM\left(0, A\left(\frac{\bar{r}_t s_{r,t}}{2 \sigma^2_{x,t}}\right) A(\nu_b)\right)$. However, apart from the direct measurement model, the equivalent noise $\nu'_{\theta,t}$ is not independent of the orientation estimate $\hat{\theta}_t$, and applying direct observation update leads to over-confidence problem. Instead of applying observation update formula, we replace the orientation estimates by the equivalent orientation measurement, distributed by $vM\left(\tan^{-1} \left(\frac{y_l - \bar{y}_t}{x_l - \bar{x}_t} \right) - s_{b,t}, A\left(\frac{\bar{r}_t s_{r,t}}{2 \sigma^2_{x,t}}\right) A(\nu_b)\right)$.

\begin{algorithm}[t]
  \caption{Localization Algorithm of Mixture Representation with Bearing and Distance Measurements}
  \begin{algorithmic}
    \State \textbf{Initialization}
    \State\quad Set $\bar{\theta}_0, \kappa_t$ for orientation estimation.
    \State\quad Set $\bar{x}_0,\bar{y}_0, \sigma_{x,0}^2,\sigma_{y,0}^2$ for position estimation.
    
    \State \textbf{Time update}%%%%%%%%%%%%%%%%%
    \State\quad input: odometry input $\omega_t, v_{t}$
    \[
       (\bar{\theta}_{t+1},\kappa_{t+1})= \texttt{vMF\_Time}\left(\bar{\theta}_{t},\kappa_{t};  \omega_t \Delta t, \frac{1}{ \sigma_{\omega}^2 \Delta t^2}\right).
    \]
    \begin{align*}
       (\bar{x}_{t+1}, \sigma_{x,t+1}^2) &= \texttt{KF\_Time}\big(\bar{x}_{t}, \sigma_{x,t}^2;\\
       &\quad\quad v_t \Delta t  A(\kappa_t) \cos\theta_t, (\sigma_v^2 + v_t^2)\Delta t ^2\big), \\
       (\bar{y}_{t+1}, \sigma_{y,t+1}^2) &= \texttt{KF\_Time}\big(\bar{y}_{t}, \sigma_{y,t}^2;\\
       &\quad\quad v_t \Delta t  A(\kappa_t) \cos\theta_t, (\sigma_v^2 + v_t^2)\Delta t ^2\big).
    \end{align*}

    \State \textbf{Observation update}%%%%%%%%%%%%%%%%%%%
    \State\quad input: bearing measurement $s_{b,t}$ and distance measurement $s_{r,t}$
    \begin{align*}
       \bar{\theta}_{t^+} &= \tan^{-1} \left(\frac{y_l - \bar{y}_t}{x_l - \bar{x}_t} \right) - s_{b,t}, \\
       \kappa_{t^+} &= A\left(\frac{\bar{r}_t s_{r,t}}{2 \sigma^2_{x,t}}\right) A(\nu_b).
    \end{align*}
    \begin{align*}
       (\bar{x}_{t^+}, \sigma_{x,t^+}^{2}) &= \texttt{KF\_Obsv}\big(\bar{x}_{t}, \sigma_{x,t}^{2};   \\
       & x_l -s_{r,t}A(\kappa_t)A(\kappa_b)\cos(\bar{\theta}_t + s_{b,t}), \sigma_r^2 + s_{r,t}^2  \big), \\
       (\bar{y}_{t^+}, \sigma_{y,t^+}^{2}) &= \texttt{KF\_Obsv}\big(\bar{y}_{t}, \sigma_{y,t}^{2};   \\
       & y_l -s_{r,t}A(\kappa_t)A(\kappa_b)\sin(\bar{\theta}_t + s_{b,t}), \sigma_r^2 + s_{r,t}^2  \big).
    \end{align*}
    \\\hrulefill
    \State \textbf{Kalman filter}
    \State Time update: $(\bar{s}_{t+1}, \sigma^2_{t+1}) = \texttt{KF\_Time}(\bar{s}_{t}, \sigma^2_{t}; u_t, \sigma^2_{w})$
    \begin{align*}
       \bar{s}_{t+1} &= \bar{s}_{t} + u_t, \\
       \sigma^2_{t+1} &= \sigma^2_{t} + \sigma^2_{w}.
    \end{align*}
    \State Observation update: $(\bar{s}_{t^+}, \sigma^2_{t^+}) = \texttt{KF\_Obsv}(\bar{s}_{t}, \sigma^2_{t}; o_t, \sigma^2_{r})$
    \begin{align*}
       k_t &= \sigma^2_{t} \left( \sigma^2_{r} + \sigma^2_{t} \right)^{-1}, \\
       \bar{s}_{t^+} &= \bar{s}_{t} + k_t (o_t - \bar{s}_t), \\
       \sigma^2_{t^+} &= \left( \sigma^{-2}_{t} + \sigma^{-2}_{r} \right)^{-1}.
    \end{align*}
  \end{algorithmic}
\end{algorithm}

   % position
   As for the position update, the distribution of equivalent direction position measurements can be constructed similarly. For example, for the equivalent observation $o'_{x,t}$,
   \begin{align*}
      \E[o'_{x,t}] &= x_l - s_{r,t} A(\kappa_t)A(\kappa_b) \cos(\bar{\theta}_t + s_{b,t}), \\
      \var(o'_{x,t}) &\leq \sigma_r^2 + s_{r,t}^2.
   \end{align*}
   Consequently, the observation gives the condition that
   \[
      \hat{x}_{t^-} =  x_l - s_{r,t} A(\kappa_t)A(\kappa_b) \cos(\bar{\theta}_t + s_{b,t})
   \]
   accompanied with the noise $\nu'_{x,t}$ approximated by $N(0,\sigma_r^2 + s_{r,t}^2)$. Note that by taking the variance as $\sigma_r^2 + s_{r,t}^2$, the effect from the variance of orientation estimate $\hat{\theta}_t$ is compensated and vanished. As a consequence, the observation update based on direct measurement follows.
   The estimate $\hat{y}_t$ can be updated in the identical method. The entire algorithm is summarized in Algorithm 2.

   For the localization algorithms with bearing-and-distance measurement, the Jacobin is essential in EKF for linear approximation. But in real implementation the Jacobian can only be calculated with estimated values rather than the ground truth values. Therefore, the estimation error in EKF comes from both the linearization and the discrepancy between the estimated and ground truth values, which jointly undermines the localization performance. On the contrary, the proposed algorithms here rely on the approximation between von Mises and Gaussian distributions while maintaining the first two moments of the estimates.

%%%%%
\section{Grid Cells and Localization Algorithm with Fully Circular Representation}

   Even though the firing pattern of a single grid cell is periodic and does not give exact positional information, an ensemble of grid cells with different spatial scales does provide accurate position representation. As a result, circular representation of the spatial state is ubiquitous in mammalian brains, with head direction cells for orientation and grid cells for position. In this section, we introduce the current understanding of grid cells, and propose the corresponding localization algorithm based on fully circular representation.

%%%
\subsection{Spatial Representation by Grid Cells}
   
   % single module
   Anatomically, proximate grid cells have the same spatial scale and orientation but different spatial phases, and those grid cells with identical spatial scale are organized as a module. According to \cite{mathis_multiscale_2013}, the collection of grid cells in the same module reduces the intrinsic representation uncertainty of each neuron, and signifies a single circular position of that module, as an example of population vector decoding \cite{thomas_trappenberg_fundamentals_2002}.  Furthermore, the spatial periods in each module, denoted by $\lambda_1,\dots,\lambda_M$, form a discrete set with a relatively constant ratio $3/2$ between adjacent modules. In \cite{banino_vector_based_2018}, the experiment on artificial agents shows similar result on the ratio between module scales. Several works try to demystify the constant spatial scale ratio as the result of minimization the number of neurons required for a given resolution \cite{wei_principle_2015}, or as the prevention of large-scale representation error \cite{stemmler_connecting_2015}.

   % from 1d to 2d   
   In terms of construction, the extension to two-dimensional circular representation is straightforward with the setup of one-dimensional circular representation. However, direct assigning two independent one-dimensional circular representation to $x$- and $y$-axes in Cartesian coordinate system leads to a rectangular grid, rather than hexagonal grid, which appears when two axes evolve along a twisted torus. For simplicity, we consider the Cartesian case here, and leave the twisted torus implementation for further research.

%%%
\subsection{Localization with Fully Circular Representation}

   The key of circular representation of position relies on several circular estimates with various spatial scales as explained. To be specific, the circular representation of $x$-axis can be achieved by multiple phases $\phi_{1,t},\dots,\phi_{M,t}$, where each phase is associated with spatial period $\lambda_{1},\dots,\lambda_{M}$, respectively. The $y$-axis position can be represented by $\psi_{1,t},\dots,\psi_{M,t}$ similarly. With the spatial state fully represented in circular form, the estimates of the state, including $\hat{\theta}_t,\hat{\phi}_{1,t},\dots,\hat{\phi}_{M,t}, \hat{\psi}_{1,t},\dots,\hat{\psi}_{M,t}$, now all follow von Mises distribution, whose update relies on Algorithm 1. In addition to the estimate update, the conversion from circular to Cartesian representation is also essential for further application that needs explicit position information. 
   
   % time update
   As for the localization algorithm with circular representation, since the orientation estimate $\hat{\theta}_t$ is identical to the mixture representation case, we only discuss the position estimates in the following.
   With angular velocity $\omega_t$ and translational velocity $\nu_t$ as input in time update, we can rewrite (\ref{eq:l_time_model}) in circular form as
   \begin{align}
      \phi_{i,t+1} &= \phi_{i,t} + \frac{2 \pi}{\lambda_i} (v_t + n_{v,t})\cos\theta_t \Delta t, \notag\\
      \psi_{i,t+1} &= \psi_{i,t} + \frac{2 \pi}{\lambda_i} (v_t + n_{v,t})\sin\theta_t \Delta t. \label{eq:l_time_model_circular}
   \end{align}
   By characterizing the first two moments of $(v_t + n_{v,t}) \cos\hat{\theta}_t$ as in the previous section, the time update equation can be expressed in the form compatible with (\ref{eq:time_model}) as
   \begin{equation}
      \phi_{i,t+1} = \phi_{i,t} + u_{\phi,i,t} + w_{\phi,i,t},
   \end{equation}
   where 
   \[
      u_{\phi,i,t} = \frac{2 \pi}{\lambda_i} v_t A(\kappa_t)\cos\bar{\theta}_t \Delta t
   \]
   and $w_{\phi,i,t}$ follows $vM(0,\kappa_{w_{\phi,i}})$ with
   \[
      \kappa_{w_{\phi,i}} = \frac{\lambda_i^2}{4\pi^2 (\Delta t)^2 (\sigma_v^2 + v_t^2)}.
   \]
   The procedure can be applied for $y$-axis update of $\hat{\psi}_{1,t},\dots,\hat{\psi}_{M,t}$ as well. Since all spatial states are in circular form, the time update can be done by Algorithm 1 with respective parameters.

   % direct measurement update
   In terms of the direct measurement model (\ref{eq:d_observ_model}), the observation update for $\hat{\theta}_t$ is direct. Those for $\hat{\phi}_{1,t},\dots,\hat{\phi}_{M,t}, \hat{\psi}_{1,t},\dots,\hat{\psi}_{M,t}$ can also be realized with appropriate von Mises approximation. That is, by converting the direct $x$-position observation $o_{x,t}$ in (\ref{eq:d_observ_model}) into circular form as
   \begin{equation}
      o_{\phi,i,t} = \phi_{i,t} + \nu_{\phi,i,t},
   \end{equation}
   where $\nu_{\phi,i,t}$ follows $vM(0, \lambda_i^2/4 \pi^2 \sigma_{o_x}^2)$, the updates of circular represented position updates simply follow. The exact procedure can be applied for $y$-axis position phases for certain. The observation update can be extended to handle bearing and distance measurements with the approximation between Gaussian and von Mises distributions. As for the uncertainty of position estimates, we take the reciprocal of the concentration parameter of the largest spatial scale in the module as the approximated variance.

\begin{algorithm}[t]
  \caption{Localization Algorithm of Fully Circular Representation with Bearing and Distance Measurements}
  \begin{algorithmic}
    \State \textbf{Initialization}
    \State\quad Set $\bar{\theta}_0, \kappa_t$ for orientation estimation.
    \State\quad Set $\bar{\phi}_{0,0},\dots,\bar{\phi}_{M,0}$ and $\kappa_{\phi,0,0}, \dots, \kappa_{\phi,M,0}$ for $x$ position estimation.
    \State\quad Set $\bar{\psi}_{0,0},\dots,\bar{\psi}_{M,0}$ and $\kappa_{\psi,0,0}, \dots, \kappa_{\psi,M,0}$  for $y$ position estimation.
    
    \State \textbf{Time update}%%%%%%%%%%%%%%%%%
    \State\quad input: odometry input $\omega_t, v_{t}$
    \[
       (\bar{\theta}_{t+1},\kappa_{t+1})= \texttt{vMF\_Time}\left(\bar{\theta}_{t},\kappa_{t};  \omega_t \Delta t, 1/ \sigma_{\omega}^2 \Delta t^2\right).
    \]
    \begin{align*}
       \kappa_{w_{\phi,i}} &= \kappa_{w_{\psi,i}} = \lambda_i^2 / 4\pi^2 (\Delta t)^2 (\sigma_v^2 + v_t^2), \\
       u_{i,t} &= 2 \pi v_t A(\kappa_t)\Delta t / \lambda_i, \\
      (\bar{\phi}_{i,t+1}, \kappa_{\phi,i,t+1}) &= \texttt{vMF\_Time}\left( \cdot, \cdot; u_{i,t}\cos\bar{\theta}_t, \kappa_{w_{\phi,i}} \right),\\
      (\bar{\psi}_{i,t+1}, \kappa_{\psi,i,t+1}) &= \texttt{vMF\_Time}\left(  \cdot, \cdot; u_{i,t}\sin\bar{\theta}_t, \kappa_{w_{\psi,i}} \right).
    \end{align*}

    \State \textbf{Observation update}%%%%%%%%%%%%%%%%%%%
    \State\quad readout: $\bar{x}_{t}, \bar{y}_t$ are given by (\ref{eq:readout})
    \State\quad input: bearing measurement $s_{b,t}$ and distance measurement $s_{r,t}$
    \begin{align*}
       \bar{\theta}_{t^+} &= \tan^{-1} \left(\frac{y_l - \bar{y}_t}{x_l - \bar{x}_t} \right) - s_{b,t}, \\
       \kappa_{t^+} &= A\left(\frac{\bar{r}_t s_{r,t}}{2 \sigma^2_{x,t}}\right) A(\nu_b).
    \end{align*}
    \begin{align*}
      o_{\phi,i,t} &= \left[x_l -s_{r,t}A(\kappa_t)A(\kappa_b)\cos(\bar{\theta}_t + s_{b,t}) \right]/ \lambda_i,\\
      o_{\psi,i,t} &= \left[y_l -s_{r,t}A(\kappa_t)A(\kappa_b)\sin(\bar{\theta}_t + s_{b,t}) \right]/ \lambda_i,\\
      (\bar{\phi}_{i,t^+}, \kappa_{\phi,i,t^+}) &= \texttt{vMF\_Obsv}\Big( \cdot, \cdot; o_{\phi,i,t}, \lambda_i^2/4 \pi^2 \sigma_{o_x}^2\Big),\\
      (\bar{\psi}_{i,t^+}, \kappa_{\psi,i,t^+}) &= \texttt{vMF\_Obsv}\Big( \cdot, \cdot; o_{\psi,i,t},\lambda_i^2/4 \pi^2 \sigma_{o_y}^2 \Big).
    \end{align*}
  \end{algorithmic}
\end{algorithm}

   % conversion
   We now turn to the conversion from the fully circular representation to conventional Cartesian coordinate, which is known as readout algorithm and discussed extensively in \cite{bush_using_2015}. To begin with, we take $x$-axis for example. Since each $x$-axis estimate $\hat{\phi}_{i,t}$ gives a most likely position with $\lambda_i$ spatial period, for $i=1,\dots, M$, we then pick the $x$-position estimate $\bar{x}_t$ as the position that maximize the overall likelihood, as
   \begin{equation}
      \bar{x}_t = \argmax_{x \in C_x} \sum_{i=1}^M \kappa_i \cos\left( \frac{2\pi}{\lambda_i} x - \bar{\phi}_{i,t} \right),
      \label{eq:readout}
   \end{equation}
   where $C_x$ indicates the coverage that the circular representation is valid.

\section{Simulation}

   \begin{figure}[t]
      \centering
      \includegraphics[width=0.4\textwidth]{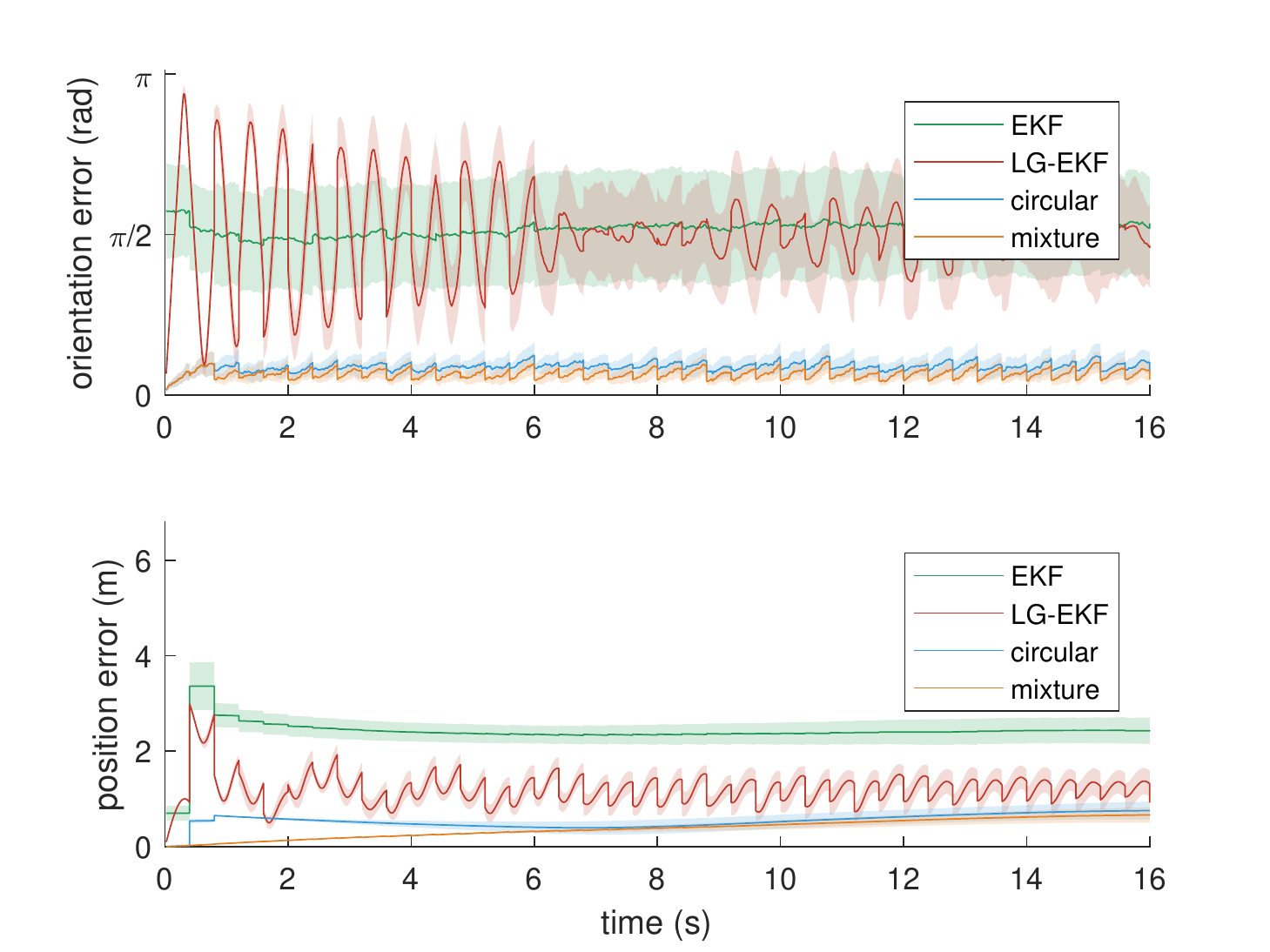}
      \caption{The spatial state estimation errors with bearing and distance measurements over $50$ trials. The shadowed area indicates the region within $1$ standard deviation. EKF and LG-EKF do not maintain accurate orientation estimation, and EKF also shows significant error variation. Even though LG-EKF improves the position estimation over EKF, it still worse than the performance of the proposed algorithms.}
      \label{fig:sim}
   \end{figure}

   % algorithm setup
   We simulate the proposed algorithms, both mixture and fully circular representations, with existing localization algorithms based on EKF and LG-EKF. For fully circular representation, we consider $M=4$ modules for position estimates and the smallest spatial scale of those modules is $2.5$ m, while the scales of the rest modules are multiplied by constant ratio $3/2$. We also set $C_x=C_y=[-5, 5]$. The localization algorithm based on LG-EKF \cite{bourmaud_discrete_2013} considers the state 
   \[
      S_t = \begin{bmatrix}
      \cos\theta_t & -\sin\theta_t & x_t \\
      \sin\theta_t & \cos\theta_t & y_t \\
      0 & 0 & 1
      \end{bmatrix}
      \exp(\epsilon_t) \in SE(2), 
   \]
   where $\epsilon_t \in so(2)$ can be seen as a Lie algebraic error with zero mean.
   For all algorithms, the time update is performed at $50$ Hz. The odometry inputs are constant with $\omega_t=0.2$ rad/s and $v_t=0.1$ m/s, and the noise parameters are chosen as $\sigma_{\omega}^2 \Delta t^2 = 0.004$ and $\sigma_{v}=0.01$ m/s. The agent can observe a landmark situated at $(2,3)$ m at $2.5$ Hz via bearing and distance measurements, whose noise parameters are given by $\kappa_b=500$ and $\sigma_r=0.01$ m.

   % discussion
   
   We plot both the mean of orientation error
   \[
      \tilde{\theta}_t = |\bar{\theta}_t - \theta_t|,
   \]
   and that of the position error
   \[
      \tilde{p}_t = \sqrt{(\bar{x}_t - x)^2 + (\bar{y}_t - y_t)^2}
   \]
   over $50$ trials in Fig. \ref{fig:sim}. The shadowed area indicates the region within $1$ standard deviation over $50$ trials for each algorithms. It is obvious that both orientation estimation errors in EKF and in LG-KEF are severe. The large standard deviation of EKF orientation error shows that EKF may lead to very inaccurate orientation estimate from time to time, which deteriorates the overall spatial state estimate consequently. LG-EKF alleviates the position estimation error comparing to EKF, but the performance is still worse than that of the proposed algorithms. The estimation results of the proposed algorithms, with mixture representation and with fully circular representation, are accurate in both orientation and position, which demonstrates the effectiveness of the circular representation in localization.   
   
\section{Discussion}

   % engineering, algorithm itself
   In the proposed algorithms, the correlations between state estimates are not tracked and utilized, which leads to the two-stage observation update with bearing and distance measurements. By derivation, the estimates in the proposed algorithms are conservative in order to avoid over-confidence problem, and the utilization of correlation will further improve the estimation accuracy. In addition, the theoretical analysis of the proposed algorithms, especially the boundedness of uncertainty, is also essential to complete the  understanding of these algorithms. 
   
   %As for readout function, the accuracy not only depends on the number of estimates in a module $M$ and the spatial periods $\lambda_1,\dots,\lambda_M$, but also on the representation accuracy of $\bar{\phi}_{i,t}$. We expect thorough investigation to completely understand the conversion between two representations.

   % neuroscience
   This paper aims to enhance the cross-talk between robotics and neuroscience to benefit both fields. For instance, in \cite{fiete_what_2008}, the authors suggest the implementational superiority of circular representation over fixed-based numerical system in terms of carry-free arithmetic property and narrow register range, which is also applicable for robotic design.
   In addition to the similarity of state representation discussed in this paper, we expect more implication from neuroscience to robotics, especially the architecture of autonomous systems. Different from robotic approaches that build autonomous systems from scratch, neuroscientists study those agents that already integrate various intelligent tasks, including localization, map construction, navigation, memory and learning. In particular, on top of the spatial structure represented by grid cells, the environment information is believed to be encoded by place cells in hippocampus. In this sense, place cells provides exteroceptive information as input to grid cells, just as observation update in the proposed algorithm \cite{guanella_model_2007}. Furthermore, with episodic memory mainly stored in hippocampus, place cells support mind-travel to preview places to visit in advance, or the path planning in robotic terminology \cite{sanders_grid_2015}. Neuroscience provides abundant reference for designing robots, and ultimately approaches in robotics and neuroscience should merge into a unified theory to provide and to explain  the sense of space in autonomous agents.

\section{Conclusion}

   EKF is extensively used in robotics to conduct nonlinear estimation with localization as a representative example, but the error stemming from arbitrary linearization shadows its applicability. In this paper, without being confined to Gaussian distribution only, we propose a localization algorithm with mixture representation for spatial state, and are able to characterize the first two moments of the estimates in nonlinear models. We further extend the algorithm to fully circular representation, which is similar to the grid patterns in mammalian brains and in recurrent neural networks. Along with the discovery of grid patterns in distinct autonomous agents, this paper provides a  theoretical perspective to deepen the knowledge of such grid patterns. While this work only deals with 2D localization, we look forward to continuing the investigation to SLAM and to 3D scenarios in the future.

\section*{Appendix I: Properties of $A(\kappa)$}

\begin{lemma}[Modified Bessel functions \cite{laforgia_inequalities_2010}]
   For real $p \geq 0$,
   \begin{equation}
      \frac{-p+\sqrt{p^2+\kappa^2}}{\kappa} < \frac{I_p(\kappa)}{I_{p-1}(\kappa)}.
   \end{equation}
   Also, for $p \geq 1/2$, the inequality $I_p(\kappa)/I_{p-1}(\kappa)<1$ holds.
   \label{lemma:mbf}
\end{lemma}

   Theorem \ref{theorem:func_A} is a direct result of the properties of $A(\kappa)$, and can be proved with the following lemma.
\begin{lemma}
   For $\kappa > 0$,
   \begin{enumerate}
      \item $0  < A(\kappa) < 1$,
      \item $0 < A'(\kappa)$.
   \end{enumerate}
   \label{lemma:func_A}
\end{lemma}
\begin{proof}
   By Lemma \ref{lemma:mbf}, we have $A(\kappa) < 1$ and for $\kappa> 0$,
   \begin{equation}
      A(\kappa) > \frac{\sqrt{\kappa^2+1}-1}{\kappa} > 0,
      \label{apdx:inequality}
   \end{equation}
   which gives the lower bound of $A(\kappa)$. 
   In \cite{mardia_directional_1999}, we have the recurrent equation
   \begin{equation}
      A'(\kappa) = 1 - A(\kappa) \left(  A(\kappa) + \frac{1}{\kappa}\right).
      \label{apdx:derivative}
   \end{equation}
   Together with (\ref{apdx:inequality}), the inequality for $A'(\kappa)$ is proved.
\end{proof}

   The results from Lemma \ref{lemma:func_A} state that
   $
      A(\kappa_1)A(\kappa_2) < A(\min(\kappa_1,\kappa_2))$,
   which leads to Theorem \ref{theorem:func_A} directly.

%\begin{lemma}[Stronger version]
%   For $\kappa > 0$,
%   \begin{enumerate}
%      \item $0  < A(\kappa) < 1$,
%      \item $0 < A'(\kappa) < 1/2$,
%      \item $A''(\kappa) < 0$.
%   \end{enumerate}
%\end{lemma}

\section*{Appendix II: Properties of Positive Definite Matrices}

\begin{lemma}  %% partition lemma
   Suppose that $P > 0$ and 
   \[
      P = 
      \begin{bmatrix}
         A_{m \times m} & B\\
         B^{\T}  & C_{n \times n}
      \end{bmatrix},
   \]
   then for $c\in(0,1)$,
   \[ P < P'=
      \begin{bmatrix}
         \frac{1}{c}A & 0\\
         0  & \frac{1}{1-c}C
      \end{bmatrix}.
   \]
   \label{lemma:diagonal}
\end{lemma}
\begin{proof}
   Consider a nonzero vector $u^{\T} =[ x^{\T}, \, y^{\T} ] $ where $x \in \mathbb{R}^m$ and $y \in \mathbb{R}^n$. We can construct $v^{\T} =[(1-c) x^{\T}, \, c y^{\T} ] $. Since $P > 0$, $v^{\T}P v > 0$, or $u^{\T} Q u >0$, where
   \[ Q=
      \begin{bmatrix}
         (1-c)^2 A & c(1-c)B\\
         c(1-c)B^{\T}  & c^2 C
      \end{bmatrix} > 0.
   \]
   The result then follows by
   \begin{align*}
      P' - P &=
      \begin{bmatrix}
         \frac{1-c}{c} A & -B\\
         -B^{\T} & \frac{c}{1-c} C
      \end{bmatrix}    \\  
      &= \frac{1}{c(1-c)}
      \begin{bmatrix}
         I_{m}  & 0\\
         0 & -I_{n}
      \end{bmatrix}
      Q
      \begin{bmatrix}
         I_{m}  & 0\\
         0 & -I_{n}
      \end{bmatrix}      > 0.
    \end{align*}
\end{proof}

%\section*{Ackowledgement}

%The preferred spelling of the word ÒacknowledgmentÓ in America is without an ÒeÓ after the ÒgÓ. Avoid the stilted expression, ÒOne of us (R. B. G.) thanks . . .Ó  Instead, try ÒR. B. G. thanksÓ. Put sponsor acknowledgments in the unnumbered footnote on the first page.

%\nocite{*}
\bibliographystyle{IEEEtran}
{\footnotesize \bibliography{ref}}

\end{document}